\documentclass{article} % For LaTeX2e
\usepackage{iclr2016_conference,times}
\usepackage{hyperref}
\usepackage{url}
\usepackage{times}
\usepackage{helvet}
\usepackage{courier}
\usepackage{epstopdf}
\usepackage{amsmath}
\usepackage{amsthm}
\usepackage{amssymb}
\usepackage{braket}
\usepackage{mathtools}
\usepackage{enumitem}
\usepackage{graphicx}
\usepackage{natbib}
\usepackage{mathrsfs}
\usepackage{caption}
\usepackage{subcaption}
\usepackage[english]{babel}
\input cyracc.def
\font\tencyr=wncyr10
\def\cyr{\tencyr\cyracc}

\title{Data Representation and Compression Using Linear-Programming Approximations}

\author{
Hristo S. Paskov \\
Computer Science Department\\
Stanford University\\
\texttt{hpaskov@stanford.edu}
\And
John C. Mitchell \\
Computer Science Department\\
Stanford University\\
\texttt{jcm@stanford.edu}
\And
Trevor J. Hastie \\
Statistics Department\\
Stanford University\\
\texttt{hastie@stanford.edu}
}

\newtheorem{prop}{Proposition}
\newtheorem{thm}{Theorem}
\newtheorem{lem}{Lemma}

\DeclareMathOperator*{\minimize}{minimize}
\newcommand{\ones}{\mathbf{1}}
\newcommand{\zeros}{\mathbf{0}}

\newcommand{\mydim}{\text{dim}\,}
\newcommand{\myint}{\text{int}\,}
\newcommand{\myrelint}{\text{relint}\,}
\newcommand{\mysup}{\text{sup}\,}
\newcommand{\myinf}{\text{inf}\,}

\iclrfinalcopy % Uncomment for camera-ready version

\begin{document}

\maketitle

\begin{abstract}
We propose `Dracula', a new framework for unsupervised feature selection from sequential data such as text. Dracula learns a dictionary of $n$-grams that efficiently compresses a given corpus and recursively compresses its own dictionary; in effect, Dracula is a `deep' extension of Compressive Feature Learning. It requires solving a binary linear program that may be relaxed to a linear program. Both problems exhibit considerable structure, their solution paths are well behaved, and we identify parameters which control the depth and diversity of the dictionary. We also discuss how to derive features from the compressed documents and show that while certain unregularized linear models are invariant to the structure of the compressed dictionary, this structure may be used to regularize learning. Experiments are presented that demonstrate the efficacy of Dracula's features.
\end{abstract}

\section{Introduction}
\label{sec:intro}
At the core of any successful machine learning problem is a good feature representation that highlights salient properties in the data and acts as an effective interface to the statistical model used for inference. This paper focuses on using classical ideas from compression to derive useful feature representations for sequential data such as text. The basic tenets of compression abound in machine learning: the minimum description length principle can be used to justify regularization as well as various model selection criteria (\cite{gabrilovich2004text}), while for unsupervised problems deep autoencoders (\cite{Salakhutdinov09learningdeep}) and the classical K-means algorithm both seek a parsimonious description of data. Meanwhile, off-the-shelf compressors, such as LZ-77 (\cite{ziv1977universal}), have been successfully applied to natural language problems as kernels that compute pairwise document similarities (\cite{bratko2006spam}).

We propose a new framework, \emph{Dracula}, so called because it simultaneously finds a useful \emph{data representation and compression using linear-programming approximations} of the criterion that motivates dictionary-based compressors like LZ-77 (\cite{ziv1977universal}). Dracula finds an \emph{explicit} feature representation for the documents in a corpus by learning a dictionary of $n$-grams that is used to losslessly compress the corpus. It then recursively compresses the dictionary. This recursion makes Dracula a deep extension of Compressive Feature Learning (CFL) (\cite{paskovCFL}) that can find \emph{exponentially} smaller representations and promotes similar $n$-grams to enter the dictionary. As noted in (\cite{paskovCFL}), feature representations derived from off-the-shelf compressors are inferior because the algorithms used are sensitive to document order; both Dracula and CFL are invariant to document order.

Our framework is expressed as a binary linear program (BLP) that can viewed as a linear program (LP) over a sufficiently constrained polyhedron or relaxed to an LP by relaxing the integrality constraints. This is a notable departure from traditional deep learners (\cite{Salakhutdinov09learningdeep,DBLP:conf/naacl/SocherM13,lecun-06}), which are formulated as non-convex, non-linear optimization problems. This structure makes it possible to analyze Dracula in view of well known techniques from convex analysis (e.g. the KKT conditions), polyhedral combinatorics, and graph theory. For example, we show that Dracula is easily parameterized to control the depth and diversity of its dictionary and that its solutions are well behaved as its parameters vary.

This paper introduces Dracula in section \ref{sec:drac} and discusses some of its problem structure and computational properties, including its NP-Completeness. Section \ref{sec:feats} uses Dracula's polyhedral interpretation to explore the compressed representations it finds as its storage cost model varies. It also discusses how to extract features directly from a compression and how to integrate dictionary structure into the features. Finally, section \ref{sec:exp} provides empirical evidence that deep compression finds hierarchical structure in data that is useful for learning and compression, and section \ref{sec:conc} concludes.

\section{DRaCULA}\label{sec:drac}
This section introduces Dracula by showing how to extend CFL to a deep architecture that compresses its own dictionary elements. We also show how to interpret any Dracula solution as a directed acyclic graph (DAG) that makes precise the notion of depth and provides useful statistical insights. Finally, we prove that Dracula is NP-Complete and discuss linear relaxation schemes.

\paragraph{Notation} Throughout this paper $\Sigma$ is a fixed finite alphabet and $\mathcal{C}=\{D_1,\dots,D_N\}$ is a fixed document corpus with each $D_k \in \mathcal{C}$ a string $D_k=c^k_1\dots c^k_j$ of characters $c_i^k \in \Sigma$. An $n$-gram is any substring of some $D_k$ and $\mathcal{S}$ is the set of all $n$-grams in the document corpus, including the original documents. For any $s\in\mathcal{S}$ a \emph{pointer} $p$ is a triple $p=(s,l\in\{1,\dots,|s|\},z\in\mathcal{S})$ indicating that $z=s_{l}\dots s_{l + |z| - 1}$. We say that $p$ \emph{uses} $z$ at location $l$ in $s$. Let $\mathcal{P}$ be the set of all valid pointers and for any $P\subset\mathcal{P}$ we use $P(s) = \left\{p\in P \middle\vert p = (s, l,z)\right\}$ to select pointers whose first element is $s$, e.g. $\mathcal{P} = \cup_{s \in \mathcal{S}} \mathcal{P}(s)$. Moreover, $P$ \emph{uses} $z\in\mathcal{S}$ if there is some $p\in P$ using $z$, and $P$ \emph{reconstructs} $s\in\mathcal{S}$ if every location in $s$ is covered by at least one pointer, i.e. $\cup_{(s,l,v)\in P(s)}\{l,\dots,l+|v|-1\}=\{1,\dots,|s|\}$. Conceptually, $s$ is recovered from $P$ by iterating through the $(s,l,v)\in P$ and "pasting" a copy of $v$ into location $l$ of a blank string. It will be helpful to define $\mathcal{P}_{\mathcal{C}} = \cup_{s \in \mathcal{C}} \mathcal{P}(s)$ to be the set of pointers that can only be used to reconstruct the corpus.

\subsection{CFL}\label{sec:cfl}

CFL represents document corpus $\mathcal{C}$ by storing a \emph{dictionary} $S\subset \mathcal{S}$, a set of $n$-grams, along with a \emph{pointer set} $P\subset \mathcal{P}_{\mathcal{C}}$ that only uses dictionary $n$-grams and losslessly reconstructs each of the documents in $\mathcal{C}$. Importantly, CFL stores the dictionary directly in plaintext. The overall representation is chosen to minimize its total storage cost for a given storage cost model that specifies $d_s$, the cost of including $n$-gram $s\in \mathcal{S}$ in the dictionary, as well as $c_p$, the cost of including pointer $p\in\mathcal{P}_C$ in the pointer set. Selecting an optimal CFL representation may thus be expressed as
\begin{equation}
\begin{aligned}
\underset{S \subset \mathcal{S}, P \subset \mathcal{P}_{\mathcal{C}}}{\text{minimize}}
& \sum_{p \in P} c_p + \sum_{s \in S} d_s
&& \text{subject to}
& P \ \text{reconstructs}\ D_k\ \forall D_k \in \mathcal{C};\ P \ \text{only uses} \ s\in S.
\end{aligned}
\end{equation}
This optimization problem naturally decomposes into subproblems by observing that when the dictionary is fixed, selecting the optimal pointer set decouples into $|\mathcal{C}|$ separate problems of optimally reconstructing each corpus document. We thus define the \emph{reconstruction module} for document $D_k\in \mathcal{C}$, which takes as input a dictionary $S$ and outputs the minimum cost of reconstructing $D_k$ with pointers that only use strings in $S$. Note that specific pointers and dictionary strings can be disallowed by setting their respective costs to $\infty$. For example setting $d_s = \infty$ for all $s\in\mathcal{S}$ longer than a certain length limits the size of dictionary $n$-grams. Of course, in practice, any variables with infinite costs are simply disregarded.

The reconstruction module can be expressed as a BLP by associating with every pointer $p \in \mathcal{P}(D_k)$ a binary indicator variable $w_p\in\{0,1\}$ whereby $w_p=1$ indicates that $p$ is included in the optimal pointer set for $D_k$. We similarly use binary variables $t_s \in \{0,1\}$ to indicate that $s \in \mathcal{S}$ is included in the dictionary. Since there is a one-to-one correspondence between pointer sets (dictionaries) and $w \in \{0,1\}^{|\mathcal{P}(D_k)|}$ ($t \in \{0,1\}^{|\mathcal{S}}|$), the vector storing the $w_p$ ($t_s$), we will directly refer to these vectors as pointer sets (dictionaries). Lossless reconstruction is encoded by the constraint $X^{D_k}w \geq \mathbf{1}$ where $X^{D_k} \in \{0,1\}^{|D_k|\times |\mathcal{P}(D_k)|}$ is a binary matrix indicating the indices of ${D_k}$ that each pointer can reconstruct. In particular, for every $p=({D_k},l,z) \in \mathcal{P}({D_k})$, column $X^{D_k}_p$ is all zeros except for a contiguous sequence of $1$'s in indices $l,\dots,l+|z| - 1$. Control of which pointers may be used (based on the dictionary) is achieved by the constraint $w \leq V^{D_k}t$ where $V^{D_k} \in \{0,1\}^{|\mathcal{P}({D_k})| \times |\mathcal{S}|}$ contains a row for every pointer indicating the string it uses. In particular, for every $p=({D_k},l,z)$, $V^{D_k}_{p,z} = 1$ is the only non-zero entry in the row pertaining to $p$. The BLP may now be expressed as
\begin{equation}\label{eq:doc_recon}
\begin{aligned}
R_{D_k}(t;c) = & \underset{w \in \{0,1\}^{|\mathcal{P}({D_k})|}}{\text{minimize}}
& \sum_{p \in \mathcal{P}({D_k})} w_pc_p 
&& \text{subject to}
&& X^{D_k}w \geq \mathbf{1}; \, w \leq V^{D_k} t.
\end{aligned}
\end{equation}

The optimization problem corresponding to an optimal CFL representation may now be written as a BLP by sharing the dictionary variable $t$ among the reconstruction modules for all documents in $\mathcal{C}$:
\begin{equation}\label{eq:CFL}
\begin{aligned}
\underset{t \in \{0,1\}^{|\mathcal{S}|}}{\text{minimize}}
& \sum_{{D_k} \in \mathcal{C}} R_{D_k}(t, c) + \sum_{s \in \mathcal{S}} t_sd_s
\end{aligned}
\end{equation}

\subsection{Adding Depth with DRaCULA}\label{sec:cfl_to_drac}
The simplicity of CFL's dictionary storage scheme is a fundamental shortcoming that is demonstrated by the string $aa\dots a$ consisting of the character $a$ replicated $2^{2n}$ times. Let the cost of using any pointer be $c_p = 1$ and the cost of storing any dictionary $n$-gram be its length, i.e. $d_s = |s|$. The best CFL can do is to store a single dictionary element of length $2^n$ and repeat it $2^n$ times, incurring a total storage cost of $2^{n+1}$. In contrast, a ``deep" compression scheme that recursively compresses its own dictionary by allowing dictionary strings to be represented using pointers attains \emph{exponential} space savings relative to CFL. In particular, the deep scheme constructs dictionary strings of length $2,4,\dots,2^{2n-1}$ recursively and incurs a total storage cost of $4n$ \footnote{Note that the recursive model is allowed to use pointers in the dictionary and therefore selects from a larger pointer set than CFL. Care must be taken to ensure that the comparison is fair since the ``size" of a compression is determing by the storage cost model and we could ``cheat" by setting all dictionary pointer costs to 0. Setting all pointer costs to $1$ ensures fairness. }.

Dracula extends CFL precisely in this hierarchical manner by allowing dictionary strings to be expressed as a combination of characters and pointers from shorter dictionary strings. CFL thus corresponds to a shallow special case of Dracula which only uses characters to reconstruct dictionary $n$-grams. This depth allows Dracula to leverage similarities among the dictionary strings to obtain further compression of the data. It also establishes a hierarchy among dictionary strings that allows us to interpret Dracula's representations as a directed acyclic graph (DAG) that makes precise the notion of representation depth.

Formally, a Dracula compression (compression for brevity) of corpus $\mathcal{C}$ is a triple $\mathfrak{D}=(S\subset\mathcal{S}, P\subset\mathcal{P}_{\mathcal{C}}, \hat{P}\subset\mathcal{P})$ consisting of dictionary, a pointer set $P$ that reconstructs the documents in $\mathcal{C}$, and a pointer set $\hat{P}$ that reconstructs every dictionary string in $S$. As with CFL, any pointers in $P$ may only use strings in $S$. However, a pointer $p\in \hat{P}$ reconstructing a dictionary string $s\in\mathcal{S}$ is valid if it uses a unigram (irrespective of whether the unigram is in $S$) or a \emph{proper substring} of $s$ that is in $S$. This is necessary because unigrams take on the special role of characters for dictionary strings. They are the atomic units of any dictionary, so the character set $\Sigma$ is assumed to be globally known for dictionary reconstruction. In contrast, document pointers are not allowed to use characters and may only use a unigram if it is present in $S$; this ensures that all strings used to reconstruct the corpus are included in the dictionary for use as features.

Finding an optimal Dracula representation may also be expressed as a BLP through simple modifications of CFL's objective function. In essence, the potential dictionary strings in $\mathcal{S}$ are treated like documents that only need to be reconstructed if they are used by some pointer. We extend the storage cost model to specify costs $c_p$ for all pointers $p\in\mathcal{P}_{\mathcal{C}}$ used for document reconstruction as well as costs $\hat{c}_p$ for all pointers $p\in\mathcal{P}$ used for dictionary reconstruction. In keeping with the aformentioned restrictions we assume that $\hat{c}_p = \infty$ if $p = (s,1,s)$ illegally tries to use $s$ to reconstruct $s$ and $s$ is not a unigram. The dictionary cost $d_s$ is now interpreted as the ``overhead" cost of including $s\in\mathcal{S}$ in the dictionary without regard to how it is reconstructed; CFL uses the $d_s$ to also encode the cost of storing $s$ in plaintext (e.g. reconstructing it only with characters). Finally, we introduce dictionary reconstruction modules as analogs to the (document) reconstruction modules for dictionary strings: the reconstruction module for $s \in \mathcal{S}$ takes as input a dictionary and outputs the cheapest valid reconstruction of $s$ if $s$ needs to be reconstructed. This can be written as the BLP
\begin{equation}\label{eq:dict_recon}
\begin{aligned}
\hat{R}_s(t;\hat{c}) = & \underset{w \in \{0,1\}^{|\mathcal{P}(s)|}}{\text{minimize}}
& \sum_{p \in \mathcal{P}(s)} w_p\hat{c}_p 
&& \text{subject to}
&& X^sw \geq t_s\mathbf{1}; \, w \leq \hat{V}^s t.
\end{aligned}
\end{equation}
Here $X^s$ is analogously defined as in equation (\ref{eq:dict_recon}) and $\hat{V}^s$ is analogous to $V^s$ in equation (\ref{eq:dict_recon}) except that it does not contain any rows for unigram pointers. With this setup in mind, the optimization problem corresponding to an optimal Dracula representation may be written as the BLP
\begin{equation}\label{eq:dracula}
\begin{aligned}
\underset{t \in \{0,1\}^{|\mathcal{S}|}}{\text{minimize}}
& \sum_{D_k \in \mathcal{C}} R_{D_k}(t, c) + \sum_{s \in \mathcal{S}} \left[t_s d_s+ \hat{R}_s(t;\hat{c})\right]
\end{aligned}
\end{equation}

Finally, any compression can be interpreted graphically as, and is equivalent to, a DAG whose vertices correspond to members of $\Sigma$, $S$, or $\mathcal{C}$ and whose labeled edge set is determined by the pointers: for every $(s,l,z)\in P$ or $\hat{P}$ there is a directed edge from $z$ to $s$ with label $l$. Note that $\mathfrak{D}$ defines a multi-graph since there may be multiple edges between nodes. Figure \ref{fig:dracula} shows the graph corresponding to a simple compression. As this graph encodes all of the information stored by $\mathfrak{D}$, and vice versa, we will at times treat $\mathfrak{D}$ directly as a graph. Since $\mathfrak{D}$ has no cycles, we can organize its vertices into layers akin to those formed by deep neural networks and with connections determined by the pointer set: layer $0$ consists only of characters (i.e. there is a node for every character in $\Sigma$), layer $1$ consists of all dictionary $n$-grams constructed solely from characters, higher levels pertain to longer dictionary $n$-grams, and the highest level consists of the document corpus $\mathcal{C}$. While there are multiple ways to organize the intermediate layers, a simple stratification is obtained by placing $s\in S$ into layer $i$ only if $\hat{P}(s)$ uses a string in layer $i-1$ and no strings in layers $i+1,\dots$. We note that our architecture differs from most conventional deep learning architectures which tend to focus on pairwise layer connections -- we allow arbitrary connections to higher layers.

\begin{figure}[h]
%\vspace{.3in}
%\hspace*{0cm}
\centering{\includegraphics[scale=.3]{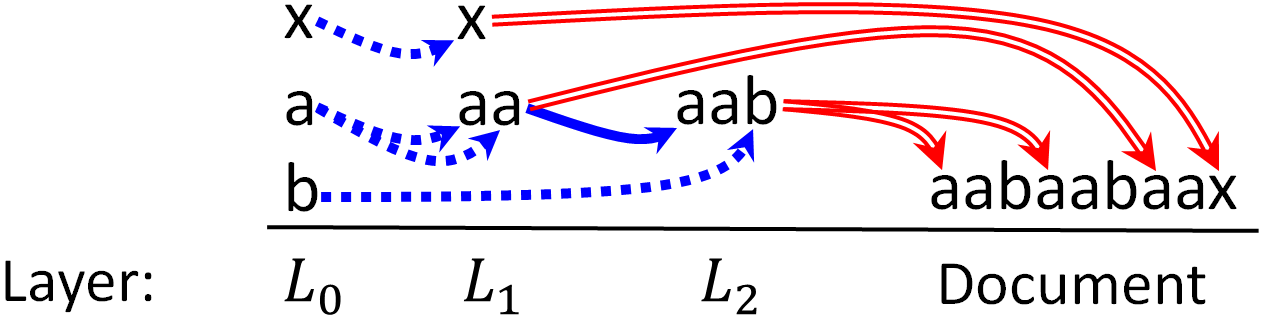}}
%\vspace{.3in}
\caption{Compression of ``aabaabaax" using a $3$-layered dictionary. Layer 0 consists of characters; layers 1 and 2 are dictionary n-grams. There are three kinds of pointers: character to dictionary $n$-gram (dashed blue lines), dictionary $n$-gram to (longer) dictionary $n$-gram (solid blue line), and dictionary $n$-gram to document (double red lines).
} \label{fig:dracula}
\end{figure}

\subsection{Computational Hardness and Relaxation}\label{sec:hardness}

The document and dictionary reconstruction modules $R_{D_k}/\hat{R}_s$ are the basic building blocks of Dracula; when dictionary $t$ is fixed, solving equation (\ref{eq:dracula}) is tantamount to solving the reconstruction modules \emph{separately}. The discussion in the Appendix section \ref{app:recon} shows that for a fixed binary $t$, solving $R_{D_k}$ or $\hat{R}_s$ is \emph{easy} because of the structure of the constraint matrices $X^{D_k}/X^s$. In fact, this problem is equivalent to a min-cost flow problem. Similarly, if the pointer sets are known for each document or dictionary string then it is easy to find the corresponding dictionary $t$ by checking which strings are used (in linear time relative to the number of pointers). One would hope that the easiness of Dracula's subproblems leads to an easy overall learning problem. However, learning the dictionary and pointer sets simultaneously makes this problem hard: Dracula is NP-Complete. In particular, it requires solving a binary LP (which are NP-Complete in general) and it generalizes CFL which is itself NP-Complete (\cite{paskovCFL2}) (see section \ref{sec:landmarks} for how to restrict representations to be shallow). 

We thus turn to solving Dracula approximately via its LP relaxation. This is obtained by replacing all binary constraints in equations (\ref{eq:doc_recon}),(\ref{eq:dict_recon}),(\ref{eq:dracula}) with interval constraints $[0,1]$. We let $\mathcal{Q}_C$ denote this LP's constraint polyhedron and note that it is a subset of the unit hypercube. Importantly, we may also interpret the original problem in equation (\ref{eq:dracula}) as an LP over a polyhedron $\mathcal{Q}$ whose vertices are always binary and hence always has binary basic solutions. Here $\mathcal{Q}$\footnote{Note that unlike $\mathcal{Q}_C$, this polyhedron is likely to be difficult to describe succinctly unless $P = NP$.} is the convex hull of all (binary) Dracula solutions and $\mathcal{Q}\subset \mathcal{Q}_C$; all valid Dracula solutions may be obtained from the linear relaxation. In fact, the Chv\'{a}tal-Gomory theorem (\cite{schrijver-book}) shows that we may ``prune" $\mathcal{Q}_C$ into $\mathcal{Q}$ by adding additional constraints. We describe additional constraints in the Appendix section \ref{app:refine} that leverage insights from suffix trees to prune $\mathcal{Q}_C$ into a tighter approximation $\mathcal{Q}_C' \subset\mathcal{Q}_C$ of $\mathcal{Q}$. Remarkably, when applied to natural language data, these constraints allowed Gurobi (\cite{gurobi}) to quickly find \emph{optimal binary solutions}. While we did not use these binary solutions in our learning experiments, they warrant further investigation.

As the pointer and dictionary costs vary, the resulting problems will vary in difficulty as measured by the gap between the objectives of the LP and binary solutions. When the costs force either $t$ or the $w^{D_k}/w^s$ to be binary, our earlier reasoning shows that the entire solution will lie on a binary vertex of $\mathcal{Q}_C$ that is necessarily optimal for the corresponding BLP and the gap will be $0$. This reasoning also shows how to round any continuous solution into a binary one by leveraging the easiness of the individual subproblems. First set all non-zero entries in $t$ to $1$, then reconstruct the documents and dictionary using this dictionary to yield binary pointers, and finally find the minimum cost dictionary based on which strings are used in the pointers.

\section{Learning with Compressed Features}\label{sec:feats}
This section explores the feature representations and compressions that can be obtained from Dracula. Central to our discussion is the observation of section \ref{sec:hardness} that all compressions obtained from Dracula are the vertices of a polyhedron. Each of these vertices can be obtained as the optimal compression for an appropriate storage cost model\footnote{The storage costs pertaining to each vertex form a polyhedral cone, see (\cite{ziegler1995lectures}) for details.}, so we take a dual perspective in which we vary the storage costs to characterize which vertices exist and how they relate to one another. The first part of this section shows how to ``walk'' around the surface of Dracula's polyhedron and it highlights some ``landmark'' compressions that are encountered, including ones that lead to classical bag-of-$n$-grams features. Our discussion applies to both, the binary and relaxed, versions of Dracula since the former can viewed as an LP over a polyhedron $\mathcal{Q}$ with only binary vertices. The second part of this section shows how to incorporate dictionary structure into features via a dictionary diffusion process.

We derive features from a compression in a bag-of-$n$-grams (BoN) manner  by counting the number of pointers that use each dictionary string or character. It will be useful to explicitly distinguish between strings and characters when computing our representations and we will use squares brackets to denote the character inside a unigram, e.g. $[c]$ . Recall that given a compression $\mathfrak{D}=(S,P,\hat{P})$, a unigram pointer in $P$ (used to reconstruct a document) is interpreted as a string whereas a unigram pointer in $\hat{P}$ is interpreted as a character. We refer to any $z\in S \cup \Sigma$ as a feature and associate with every document $D_k\in\mathcal{C}$ or dictionary string $s\in S$ a BoN feature vector $x^{D_k},x^s\in\mathbb{Z}^{|S| + |\Sigma|}_+$, respectively. Entry $x^{D_k}_z$ counts the number of pointers that use $z$ to reconstruct $D_k$, i.e. $x^{D_k}_z = |\left\{p\in P(s) \middle\vert \, p = (D_k,l,z)\right\}|$, and will necessarily have $x^{D_k}_z = 0$ for all $z\in\Sigma$. Dictionary strings are treated analogously with the caveat that if $p=(s,l,z)\in\hat{P}$ uses a unigram, $p$ counts towards the character entry $x^{D_k}_{[z]}$, not $x^{D_k}_{z}$.

\subsection{Dracula's Solution Path}\label{sec:path}
Exploring Dracula's compressions is tantamount to varying the dictionary and pointer costs supplied to Dracula. When these costs can be expressed as continuous functions of a parameter $\lambda\in[0,1]$, i.e. $\forall s\in\mathcal{S}, p \in \mathcal{P}_{\mathcal{C}}, \hat{p} \in \mathcal{P}$ the cost functions $d_s(\lambda), c_p(\lambda), \hat{c}_{\hat{p}}(\lambda)$ are continuous, the optimal solution sets vary in a predictable manner around the surface of Dracula's constraint polyhedron $\mathcal{Q}$ or the polyhedron of its relaxation $\mathcal{Q}_C$. We use $\mathscr{F}(Q)$ to denote the set of faces of polyhedron $Q$ (including $Q$), and take the dimension of a face to be the dimension of its affine hull. We prove the following theorem in the Appendix section \ref{app:path}:
\begin{thm}\label{thm:path}
Let $Q\subset \mathbb{R}^d$ be a bounded polyhedron with nonempty interior and $b:[0,1]\rightarrow\mathbb{R}^d$ a continuous function. Then for some $N\in\mathbb{Z}_+\cup \{\infty\}$ there exists a countable partition $\Gamma = \left\{\gamma_i\right\}_{i = 0}^{N}$ of $[0,1]$ with corresponding faces $F_i\in\mathscr{F}(Q)$ satisfying $F_i \neq F_{i+1}$ and  $F_{i}\cap F_{i+1} \neq \emptyset$. For all $\alpha \in \gamma_i$, the solution set of the LP constrained by $Q$ and using cost vector $b(\alpha)$ is $F_i = \text{arg min}_{x \in Q}\, x^Tb(\alpha)$. Moreover, $F_i$ never has the same dimension as $F_{i+1}$ and the boundary between $\gamma_i,\gamma_{i+1}$ is $)[$ iff $\textnormal{dim}\,F_i < \textnormal{dim}\,F_{i+1}$ and $]($ otherwise.
\end{thm}

This theorem generalizes the notion of a continuous solution path typically seen in the context of regularization (e.g. the Lasso) to the LP setting where unique solutions are piecewise constant and transitions occur by going through values of $\lambda$ for which the solution set is not unique. For instance, suppose that vertex $v_0$ is uniquely optimal for some $\lambda_0 \in [0,1)$, another vertex $v_1$ is uniquely optimal for a $\lambda_0 < \lambda_1 \leq 1$, and no other vertices are optimal in $(\lambda_0,\lambda_1)$. Then Theorem \ref{thm:path} shows that $v_0$ and $v_1$ must be connected by a face (typically an edge) and there must be some $\lambda \in (\lambda_0,\lambda_1)$ for which this face is optimal. As such, varying Dracula's cost function continuously ensures that the solution set for the binary or relaxed problem will not suddenly ``jump" from one vertex to the next; it must go through an intermediary connecting face. This behavior is depicted in Figure \ref{fig:Poly} on a nonlinear projection of Dracula's constraint polyhedron for the string ``xaxabxabxacxac''.

\begin{figure}
    \centering
    \begin{subfigure}[b]{0.4\textwidth}
        \includegraphics[scale=0.5]{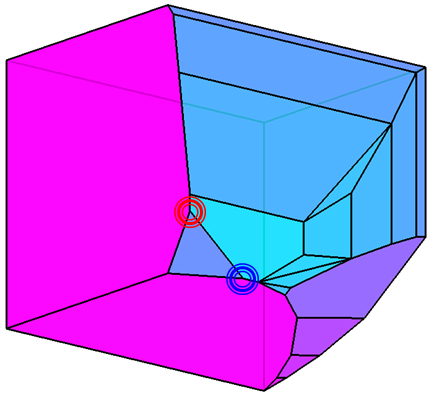}
        \caption{}
        \label{fig:primal}
    \end{subfigure}
\;\;\;\;\;
    \begin{subfigure}[b]{0.4\textwidth}
        \includegraphics[scale=0.5]{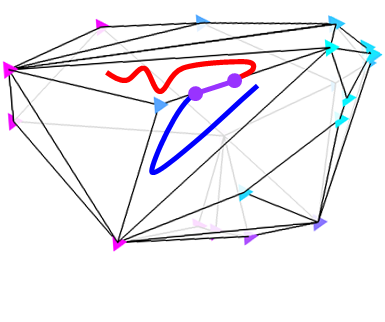}
        \caption{}
        \label{fig:dual}
    \end{subfigure}
    \begin{subfigure}[b]{\textwidth}
    \centering
        \includegraphics[scale=0.5]{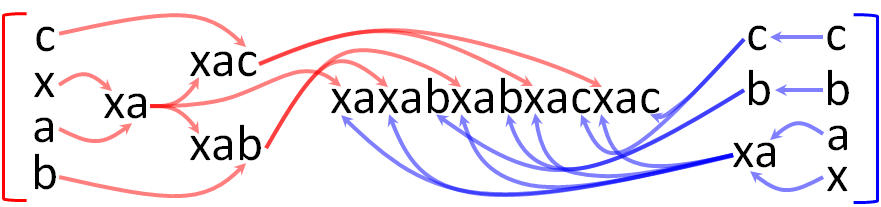}
        \caption{}
	 \label{fig:sols}
    \end{subfigure}
    \caption{Part (a) shows a nonlinear projection of a subset of Dracula's constraint polyhedron $\mathcal{Q}$ in which every vertex corresponds to a distinct compression of ``xaxabxabxacxac''. Part (b) is the projection's polar; its faces delineate the (linear) costs for which each vertex in (a) is optimal. The red/ purple/ blue line in (b) demonstrates a continuous family of costs. All red (blue) costs are uniquely minimized by the vertex in (a) highlighted in red (blue), respectively; (c) shows the corresponding compressions. Purple costs lie on the edge between the faces containing the red and blue lines and are minimized by any convex combination of the vertices highlighted in (a).}
\label{fig:Poly}
\end{figure}

It is worthwhile to note that determining the exact value of $\lambda$ for which the face connecting $v_0$ and $v_1$ is optimal is unrealistic in practice, so transitions may appear abrupt. While it is possible to smooth this behavior by adding a strongly convex term to the objective (e.g. an $L_2$ penalty), the important insight provided by this theorem is that the trajectory of the solution path depends entirely on the \emph{combinatorial structure} of $\mathcal{Q}$ or $\mathcal{Q}_C$. This structure is characterized by the face lattice\footnote{We leave it as an open problem to analytically characterize Dracula's face lattice.
} of the polyhedron and it shows which vertices are connected via edges, $2$-faces, \dots, facets. It limits, for example, the set of vertices reachable from $v_0$ when the costs vary continuously and ensure that transitions take place only along edges \footnote{Restricting transitions only to edges is possible with probability  1 by adding a small amount of Gaussian noise to $c$.}. This predictable behavior is desirable when fine tuning the compression for a learning task, akin to how one might tune the regularization parameter of a Lasso, and it is not possible to show in general for non-convex functions.

We now provide a simple linear cost scheme that has globally predictable effects on the dictionary. For all $s\in\mathcal{S}, p \in \mathcal{P}_{\mathcal{C}}, \hat{p} \in \mathcal{P}$ we set $d_s=\tau$, $c_p =1$, $\hat{c}_{\hat{p}}=\alpha\lambda$ if $\hat{p}$ uses as unigram (i.e. is a character), and $\hat{c}_{\hat{p}}=\lambda$ otherwise. We constrain $\tau,\lambda \geq 0$ and $\alpha\in[0,1]$. In words, all document pointer costs are $1$, all dictionary costs $\tau$, and dictionary pointer costs are $\lambda$ if they use a string and $\alpha\lambda$ if they use a character. The effects these parameters have on the compression may be understood by varying a single parameter and holding all others constant:
\begin{description}[style=unboxed,leftmargin=0cm]
\item [Varying $\tau$] controls the minimum frequency with which $s\in\mathcal{S}$ must be used before it enters the dictionary; if few pointers use $s$ it is cheaper to construct $s$ ``in place" using shorter $n$-grams. Long $n$-grams appear less frequently so $\uparrow \tau$ biases the dictionary towards shorter $n$-grams.
\item [Varying $\lambda$] has a similar effect to $\tau$ in that it becomes more expensive to construct $s$ as $\lambda$ increases, so the overall cost of dictionary membership increases. The effect is more nuanced, however, since the \emph{manner} in which $s$ is constructed also matters; $s$ is more likely to enter the dictionary if it shares long substrings with existing dictionary strings. This suggests a kind of grouping effect whereby groups of strings that share many substrings are likely to enter together.
\item [Varying $\alpha$] controls the Dracula's propensity to use characters in place of pointers in the dictionary and thereby directly modulates dictionary depth. When $\alpha < \frac{1}{K}$ for $K=2,3,\dots$, all dictionary $n$-grams of length at most $K$ are constructed entirely from characters.
\end{description}
\subsubsection{Landmarks on Dracula's Polyhedron} \label{sec:landmarks}
While Dracula's representations are typically deep and space saving, it is important to note that valid Dracula solutions include all of CFL's solutions as well as a set of fully redundant representations that use as many pointers as possible. The BoN features computed from these ``space maximizing'' compressions yield the traditional BoN features containing all $n$-grams up to a maximum length $K$. A cost scheme that includes all pointers using all $n$-grams up to length $K$ is obtained by setting all costs to be \emph{negative}, except for $t_s = \infty$ for all $s \in \mathcal{S}$ where $|s|> K$ (to disallow these strings). The optimal compression then includes all pointers with negative cost and each document position is reconstructed $K$ times. Moreover, it is possible to restrict representations to be valid CFL solutions by disallowing all non-unigram pointers for dictionary reconstruction, i.e. by setting $\hat{c}_p = \infty$ if $p$ is not a single character string.
\subsection{Dictionary Diffusion}\label{sec:diff}
We now discuss how to incorporate dictionary information from a compression $\mathfrak{D} = (S,P,\hat{P})$ into the BoN features for each corpus document. It will be convenient to store the BoN feature vectors $x^{D_k}$ for each document as rows in a feature matrix $X \in \mathbb{Z}^{|\mathcal{C}| \times (|S| + |\Sigma|)}$ and the BoN feature vectors $x^{s}$ for each dictionary string as rows in a feature matrix $G \in \mathbb{Z}^{(|S| + |\Sigma|) \times (|S| + |\Sigma|)}$. We also include rows of all $0$'s for every character in $\Sigma$ to make $G$ a square matrix for mathematical convenience. Graphically, this procedure transforms $\mathfrak{D}$ into a simpler DAG, $\mathfrak{D}_{\mathfrak{R}}$, by collapsing all multi-edges into single edges and labeling the resulting edges with an appropriate $x^s_z$. For any two features $s,z$, we say that $s$ is higher (lower) order than $z$ if it is a successor (predecessor) of $z$ in $\mathfrak{D}$.

Once our feature extraction process throws away positional information in the pointers higher order features capture more information than their lower order constituents since the presence of an $s\in S$ formed by concatenating features $z_1\dots z_m$ indicates the order in which the $z_i$ appear and not just that they appear. Conversely, since each $z_i$ appears in the same locations as $s$ (and typically many others), we can obtain better estimates for coefficients associated with $z_i$ than for the coefficient of $s$. If the learning problem does not require the information specified by $s$ we pay an unnecessary cost in variance by using this feature over the more frequent $z_i$. 

In view of this reasoning, feature matrix $X$ captures the highest order information about the documents but overlooks the features' lower order $n$-grams (that are indirectly used to reconstruct documents). This latter information is provided by the dictionary's structure in $G$ and can be incorporated by a graph diffusion process that propagates the counts of $s$ in each document to its constituent $z_i$, which propagate these counts to the lower order features used to construct them, and so on. This process stops once we reach the characters comprising $s$ since they are atomic. We can express this information flow in terms of $G$ by noting that the product $G^Tx^{D_k} = \sum_{s \in S\cup\Sigma}x^{D_k}_s x^s$ spreads $x^{D_k}_s$ to each of the $z_i$ used to reconstruct $s$ by multiplying $x^{D_k}_s$ with $x^s_{z_i}$, the number of times each $z_i$ is directly used in $s$. Graphically, node $s$ in $\mathfrak{D}_{\mathfrak{R}}$ sends $x^{D_k}_s$ units of flow to each parent $z_i$, and this flow is modulated in proportion to $x^s_{z_i}$, the strength of the edge connecting $z_i$ to $s$. Performing this procedure a second time, i.e. multiplying $G^T(G^Tx^{D_k})$, further spreads $x^{D_k}_sx^s_{z_i}$ to the features used to reconstruct $z_i$, modulated in proportion to their usage. Iterating this procedure defines a new feature matrix $\hat{X}=XH$ where $H=I + \sum_{n=1}^\infty G^n$ spreads the top level $x^{D_k}$ to the entire graph\footnote{This sum converges because $G$ corresponds to a finite DAG so it can be permuted to a strictly lower triangular matrix so that $\underset{n \rightarrow \infty}{\text{lim}} G^n=\mathbf{0}$. See Appendix section \ref{app:diff} for weighted variations.}.

We can interpret the effect of the dictionary diffusion process in view of two equivalent regularized learning problems that learn coefficients $\beta,\eta\in\mathbb{R}^{|S \cup \Sigma|}$ for every feature in $S \cup \Sigma$ by solving
\begin{equation}\label{eq:learn}
\begin{aligned}
& \minimize_{\beta\in\mathcal{R}^{|S \cup \Sigma|}} &&  L(\hat{X}\beta) + \lambda R(\beta) \\
\equiv & \minimize_{\eta\in\mathcal{R}^{|S \cup \Sigma|}} && L(X\eta) + \lambda R\left((I-G)\eta\right).
\end{aligned}
\end{equation}
We assume that $L$ is a convex loss (that may implicitly encode any labels), $R$ is a convex regularization penalty that attains its minimum at $\beta=\zeros$, and that a minimizer $\beta^*$ exists. Note that adding an unpenalized offset does not affect our analysis. The two problems are equivalent because $H$ is defined in terms of a convergent Neumann series and, in particular, $H=(I-G)^{-1}$ is invertible. We may switch from one problem to the other by setting $\beta=H^{-1}\eta$ or $\eta=H\beta$. 

When $\lambda=0$ the two problems reduce to estimating $\beta/\eta$ for unregularized models that only differ in the features they use, $\hat{X}$ or $X$ respectively. The equivalence of the problems shows, however, that using $\hat{X}$ in place of $X$ has \emph{no effect} on the models as their predictions are always the same. Indeed, if $\beta^*$ is optimal for the first problem then $\eta^*=H\beta^*$ is optimal for the second and for any $z\in\mathbb{R}^{|S \cup \Sigma|}$, the predictions $z^T\eta^* = (z^TH)\beta^*$ are the same. Unregularized linear models -- including generalized linear models -- are therefore \emph{invariant} to the dictionary reconstruction scheme and only depend on the document feature counts $x^{D_k}$, i.e. how documents are reconstructed.

When $\lambda > 0$, using $\hat{X}$ in place of $X$ results in a kind of graph Laplacian regularizer that encourages $\eta_s$ to be close to $\eta^Tx^s$. One interpretation of this is effect is that $\eta_s$ acts a ``label" for $s$: we use its feature representation to make a prediction for what $\eta_s$ should be and penalize the model for any deviations. A complementary line of reasoning uses the collapsed DAG $\mathfrak{D}_{\mathfrak{R}}$ to show that (\ref{eq:learn}) favors lower order features. Associated with every node $s\in S \cup \Sigma$ is a flow $\eta_s$ and node $z$ sends $\eta_z$ units of flow to each of its children $s$. This flow is attenuated (or amplified) by $x^s_z$, the strength of the edge connecting $z$ to $s$. In turn, $s$ adds its incoming flows and sends out $\eta_s$ units of flow to its children; each document's prediction is given by the sum of its incoming flows. Here $R$ acts a kind of ``flow conservation" penalty that penalizes nodes for sending out a different amount of flow than they receive and the lowest order nodes (characters) are penalized for any flow. From this viewpoint it follows that the model prefers to disrupt the flow conservation of lower order nodes whenever they sufficiently decrease the loss since they influence the largest number documents. Higher order nodes influence fewer documents than their lower order constituents and act as high frequency components.

\section{Experiments}
\label{sec:exp}

This section presents experiments comparing traditional BoN features with features derived from Dracula and CFL. Our primary goal is investigate whether deep compression can provide better features for learning than shallow compression or the traditional ``fully redundant'' BoN representation (using all $n$-grams up to a maximum length). Since any of these representations can be obtained from Dracula using an appropriate cost scheme, positive evidence for the deep compression implies Dracula is uncovering hierarchical structure which is simultaneously useful for compression and learning.  We also provide a measure of compressed size that counts the number of pointers used by each representation, i.e. the result of evaluating each compression with a ``common sense'' space objective where all costs are $1$. We use \textbf{Top} to indicate BoN features counting only document pointers ($X$ in previous section), \textbf{Flat} for dictionary diffusion features (i.e. $\hat{X}$), \textbf{CFL} for BoN features from CFL, and \textbf{All} for traditional BoN features using all $n$-grams considered by Dracula.

We used Gurobi (\cite{gurobi}) to solve the refined LP relaxation of Dracula for all of our experiments. While Gurobi can solve impressively large LP's, encoding Dracula for a general-purpose solver is inefficient and limited the scale of our experiments. Dedicated algorithms that utilize problem structure, such as the network flow interpretation of the reconstruction modules, are the subject of a follow-up paper and will allow Dracula to scale to large-scale datasets. We limited our parameter tuning to the dictionary pointer cost $\lambda$ (discussed in the solution path section) as this had the largest effect on performance. Experiments were performed with $\tau=0$, $\alpha=1$, a maximum $n$-gram length, and only on $n$-grams that appear at least twice in each corpus.

\begin{figure*}
\centering\includegraphics[scale=.45]{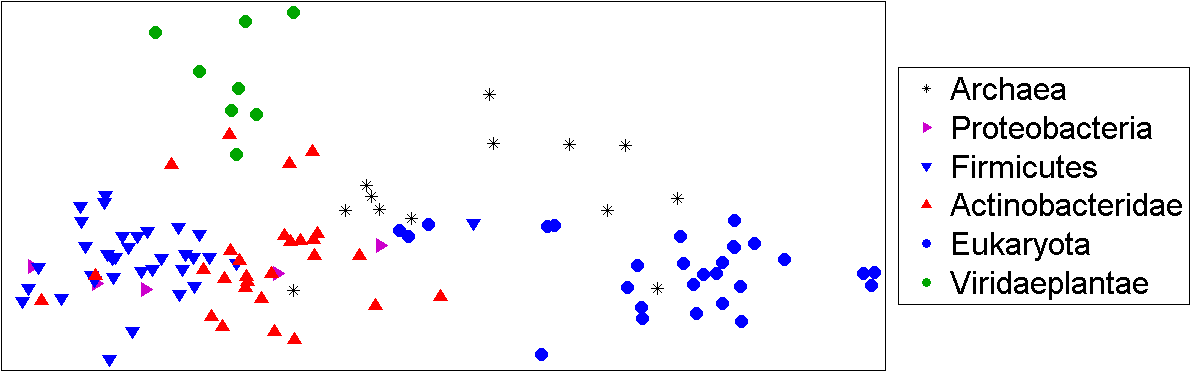}
\caption{Proteins represented using the $4^{\text{th}}$ and $5^{\text{th}}$ singular vectors of Top features from Dracula.} \label{fig:prot}
\end{figure*}

\paragraph{Protein Data} We ran Dracula using $7$-grams and $\lambda=1$ on $131$ protein sequences that are labeled with the kingdom and phylum of their organism of origin (\cite{protein}). Bacterial proteins ($73$) dominate this dataset, $68$ of which evenly come from Actinobacteria (A) and Fermicutes (F). The first $5$ singular values (SV's) of the Top features show a clear separation from the remaining SV's and Figure \ref{fig:prot} plots the proteins when represented by their $4^{\text{th}}$ and $5^{\text{th}}$ principle components. They are labeled by kingdom and, in more interesting cases, by phylum. Note the clear separation of the kingdoms, the two main bacterial phyla, and the cluster of plants separated from the other eukaryotes. Table \ref{bio-table} shows the average accuracy of two binary classification tasks in which bacteria are positive and we hold out either phylum A or F, along with other randomly sampled phyla for negative cases, as a testing set. We compare All features to Top features from Dracula and CFL using an $\ell_2$-regularized SVM with $C=1$. Since there are many more features than training examples we plot the effect of using the top $K$ principle components of each feature matrix. Flat features did not help and performance strictly decreased if we limited the $n$-gram length for All features, indicating that long $n$-grams contain essential information. Both compression criteria perform well, but using a deep dictionary seems to help as Dracula's profile is more stable than CFL's.
\begin{table}
\caption{Bacteria Identification Accuracy using Protein Data}
\begin{center}
\begin{tabular}{lllllll}\label{bio-table}
{\bf SVD Rank } & {\bf 5} & {\bf 10} & {\bf 15} & {\bf 20} & {\bf All} & {\bf \# Pointers}\\
\hline \\
\bf All & 59.5 & 77.7 & 83.3 & 77.6 & 81.1 & 4.54$\times 10^5$\\
\bf CFL & \bf 89.7 & 85.0 & 76.9 & 74.5 & 74.0 & 2.69$\times 10^4$\\
\bf Top & 87.5 & \bf 91.2 & \bf 89.0 & 83.3 & 84.3 & 1.76$\times 10^4$\\
\end{tabular}
\end{center}
\end{table}

\paragraph{Stylometry} We extracted $100$ sentences from each of the training and testing splits of the Reuters dataset (\cite{reuters}) for $10$ authors, i.e. $2,000$ total sentences, and replaced their words with part-of-speech tags. The goal of this task is to predict the author of a given set of writing samples (that all come from the same author). We make predictions by representing each author by the centroid of her $100$ training sentences, averaging together the unknown writing samples, and reporting the nearest author centroid to the sample centroid. We ran Dracula on this representation with $10$-grams and normalized centroids by their $\ell_1$ norm and features by their standard deviation. Table \ref{auth-table} compares the performance of All features to Top features derived from various $\lambda$'s for various testing sentence sample sizes. We report the average of $1,000$ trials, where each trial tested every author once and randomly selected a set of sample sentences from the testing split sentences. As in the protein data, neither Flat nor shorter $n$-gram features helped, indicating that higher order features contain vital information. CFL with $\lambda =20$ strictly dominated every other CFL representation and is the only one included for brevity. Dracula with $\lambda=10$ or $\lambda=20$ shows a clear separation from the other schemes, indicating that the deep compression finds useful structure.

\begin{table}
\caption{Author Identification Accuracy}
\begin{center}
\begin{tabular}{lllllll}\label{auth-table}
{\bf \# Samples } & {\bf 5} & {\bf 10} & {\bf 25} & {\bf 50} & {\bf 75} & {\bf \# Pointers}\\
\hline \\
\bf All & 36.0 & 47.9 & 67.9 & 80.6 & 86.4 & 5.01$\times 10^5$\\
\bf CFL $\lambda = 20$ & \bf 39.6 & 50.5 & 73.8 & 87.5 & 91.4 & 3.33$\times 10^4$\\
\bf Top $\lambda = 1$ & 35.1 & 46.2 & 68.6 & 85.3 & 93.7 & 2.39$\times 10^4$\\
\bf Top $\lambda = 10$ & \bf 39.6 & \bf 51.0 & \bf 75.0 & 88.9 & 93.7 & 3.00$\times 10^4$\\
\bf Top $\lambda = 20$ & 37.7 & 49.4 & 73.8 & \bf 91.5 & \bf 97.8 & 3.32$\times 10^4$\\
\end{tabular}
\end{center}
\end{table}

\paragraph{Sentiment Prediction} We use a dataset of $10,662$ movie review sentences (\cite{pang2005seeing}) labeled as having positive or negative sentiment. Bigrams achieve state-of-the-art accuracy on this dataset and unigrams perform nearly as well (\cite{wang2012baselines}), so enough information is stored in low order $n$-grams that the variance from longer $n$-grams hurts prediction. We ran Dracula using $5$-grams to highlight the utility of Flat features, which focus the classifier onto lower order features. Following (\cite{wang2012baselines}), Table \ref{nlp-table} compares the 10-fold CV accuracy of a multinomial na\"{\i}ve-Bayes (NB) classifier using Top or Flat features with one using all $n$-grams up to a maximum length. The dictionary diffusion process successfully highlights relevant low order features and allows the Flat representation to be competitive with bigrams (the expected best performer). The table also plots the mean $n$-gram length (\textbf{MNL}) used by document pointers as a function of $\lambda$. The MNL decreases as $\lambda$ increases and this eventually pushes the Top features to behave like a mix of bigrams and unigrams. Finally, we also show the performance of $\ell_2$ or $\ell_1$-regularized support vector machines for which we tuned the regularization parameter to minimize CV error (to avoid issues with parameter tuning). It is known that NB performs surprisingly well relative to SVMs on a variety of sentiment prediction tasks, so the dropoff in performance is expected. Both SVMs achieve their best accuracy with bigrams; the regularizers are unable to fully remove the spurious features introduced by using overly long $n$-grams. In contrast, Flat achieves its best performance with larger MNLs which suggests that Dracula performs a different kind of feature selection than is possible with direct $\ell_1/\ell_2$ regularization. Moreover, tuning $\lambda$ combines feature selection with NB or any kind of classifier, irrespective of whether it natively performs feature selection.

\begin{table}
\caption{Sentiment Classification Accuracy}
\begin{center}
\begin{tabular}{c c c c c || c c c c}\label{nlp-table}
{\bf $\lambda$: } & {\bf MNL} & {\bf \# Pointers} & {\bf Top} & {\bf Flat}  & {\bf $n$-gram Len.} & {\bf NB All} & {\bf SVM $\ell_1$ All} & {\bf SVM $\ell_2$ All}\\
\hline \\
0.25 & \bf 4.02 & 1.79$\times 10^5$ & 73.9 & 78.2 & \bf 5 & 77.9 & 76.6 & 76.9\\
0.5 & \bf 3.78 & 1.75$\times 10^5$ & 75.1 & \bf 78.8 & \bf 4 & 77.9 & 76.8 & 77.0\\
1 & \bf 3.19 & 1.71$\times 10^5$ & 76.6 & 78.2 & \bf 3 & 78.4 & 77.0 & 77.2\\
2 & \bf 2.51 & 1.71$\times 10^5$ & 78.0 & 78.1 & \bf 2 & \bf 78.8 & 77.2 & 77.5\\
5 & \bf 1.96 & 1.86$\times 10^5$ & 78.0 & 78.0 & \bf 1 & 78.0 & 76.3 & 76.5
\end{tabular}
\end{center}
\end{table}

\section{Conclusion}
\label{sec:conc}
We have introduced a novel dictionary-based compression framework for feature selection from sequential data such as text. Dracula extends CFL, which finds a shallow dictionary of $n$-grams with which to compress a document corpus, by applying the compression recursively to the dictionary. It thereby learns a deep representation of the dictionary $n$-grams and document corpus. Experiments with biological, stylometric, and natural language data confirm the usefulness of features derived from Dracula, suggesting that deep compression uncovers relevant structure in the data. 

A variety of extensions are possible, the most immediate of which is the design of an algorithm that takes advantage of the problem structure in Dracula. We have identified the basic subproblems comprising Dracula, as well as essential structure in these subproblems, that can be leveraged to scale the compression to large datasets. Ultimately, we hope to use Dracula to explore large and fundamental datasets, such as the human genome, and to investigate the kinds of structures it uncovers.

\subsubsection*{Acknowledgements}
Dedicated to {\cyr Ivan i Kalinka Handzhievi} (Ivan and Kalinka Handjievi). Funding provided by the Air Force Office of Scientific Research and the National Science Foundation.

\bibliography{bibliography}
\bibliographystyle{iclr2016_conference}

\appendix
\section{Appendix}

\subsection{Reconstruction Modules}\label{app:recon}
The reconstruction modules $R_{D_k}/\hat{R}_s$ are the basic building blocks of Dracula; when $t$ is fixed solving (\ref{eq:dracula}) is tantamount to solving the reconstruction modules \emph{separately}. These simple BLPs have a number of properties that result in computational savings because of the structure of the constraint matrix $X^{D_k}/X^s$. In order to simplify notation we define
\begin{equation}\label{eq:basic_recon_a}
\begin{aligned}
T_s(t,v;b,V) = & \underset{w \in \{0,1\}^{|\mathcal{P}(s)|}}{\text{minimize}}
& \sum_{p \in \mathcal{P}(s)} w_pb_p 
&& \text{subject to}
&& X^sw \geq v\mathbf{1}; \, w \leq V t.
\end{aligned}
\end{equation}
Using $T_{D_k}(t,1;c,V^{D_k}) = R_{D_k}(t;c)$ and $T_{s}(t,t_s;\hat{c},\hat{V}^{s}) = \hat{R}_{s}(t;\hat{c})$ results in the document or dictionary reconstruction modules. Now note that every column in $X^s$ is all zeros except for a contiguous sequence of ones so that $X^s$ is an \emph{interval matrix} and therefore totally unimodular (TUM). Define $T_s^c$ to be the LP relaxation of $T_s$ obtained by replacing the integrality constraints:
\begin{equation}\label{eq:basic_recon_b}
\begin{aligned}
T_s^c(t,v;b,V) = & \underset{w \in [0,1]^{|\mathcal{P}(s)|}}{\text{minimize}}
& \sum_{p \in \mathcal{P}(s)} w_pb_p 
&& \text{subject to}
&& X^sw \geq v\mathbf{1}; \, w \leq V t.
\end{aligned}
\end{equation}
Aside from $X$, the remaining constraints on $w$ are bound constraints. It follows from (\cite{books/daglib/0014647}) that $T_s^c$ is an LP over a integral polyhedron so we may conclude that
\begin{prop}\label{prop:1}
If the arguments $t,v$ are integral, then all basic solutions of $T_s^c(t,v;b,V)$ are binary.
\end{prop}
Indeed, a simple dynamic program discussed in (\cite{paskovCFL}) solves $T_s$ efficiently.

Our second property reformulates $T_s^c$ by transforming the constraint matrix $X^s$ into a simpler form. The resulting matrix has at most $2$ non-zero entries per column instead of up to $|s|$ non-zero entries per column in $X^s$. This form is more efficient to work with when solving the LP and it shows that $T_s^c$ is equivalent to a min-cost flow problem over an appropriately defined graph. Define $Q\in\{0,\pm1\}^{|s|\times |s|}$ be the full rank lower triangular matrix with entries $Q^s_{ii} = -Q^s_{(i+1)i}=1$ and $0$ elsewhere (and $Q^s_{|s||s|}=1$). The interval structure of $X^s$ implies that column $i$ of $Z^s=Q^sX^s$ is all zeros except for $Z^s_{ij}=-Z^s_{ik}=1$ where $j$ is the first row in which $X^s_{ij}=1$ and $k>j$ is the first row in which $X^s_{ik}=0$ after the sequences of ones (if such a $k$ exists). By introducing non-negative slack variables for the $X^sw\geq v \ones$ constraint, i.e. writing $X^sw = v \ones + \xi$, and noting that $Q^s\ones=\mathbf{e}_1$, where $\mathbf{e}_1$ is all zeros except for a $1$ as its first entry, we arrive at:

\begin{equation}\label{eq:basic_recon2_a}
\begin{aligned}
T_s^c(t,v;b,V) = \, & \underset{w, \xi}{\text{minimize}}
&& \sum_{p \in \mathcal{P}(s)} w_pb_p  \\
& \text{subject to}
&& Z^sw - Q^s\xi= v\mathbf{e}_1, \\
&&& \zeros \leq w \leq V t, \, \zeros \leq \xi.
\end{aligned}
\end{equation}
The matrix $\Psi=[Z^s \vert -Q^s]$ has special structure since every column has at most one $1$ and at most one $-1$. This allows us to interpret $\Psi$ as the incidence matrix of a directed graph if we add source and sink nodes with which to fill all columns out so that they have exactly one $1$ and one $-1$. $T_s^c$ may then be interpreted as a min-cost flow problem.

%Let $Q^s\in\{0,\pm1\}^{|s|\times |s|}$ have non-zero entries $Q^s_{ii} = -Q^s_{(i+1)i}=1, Q^s_{|s||s|}=1$, and let $Z\in\{0,\pm1\}^{|s|\times|P(s)|}$ have non-zero entries $Z^s_{ij}=-Z^s_{ik}=1$ where $j$ is the first row in which $X^s_{ij}=1$ and $k>j$ is the first row in which $X^s_{ik}=0$ after the sequence of ones (if such a $k$ exists) . Then
%\begin{equation}\label{eq:basic_recon2}
%\begin{aligned}
%R_s^c(t,v;c^s) = \, & \underset{w, \xi}{\text{min}}
%&& w^Tc^s
%& \text{subject to}
%&& Z^sw - Q^s\xi= v\mathbf{e}_1,
%&&& \zeros \leq w \leq V^s t, \, \zeros \leq \xi.
%\end{aligned}
%\end{equation}
%We note that matrix $[Z^s \; -Q^s]$ has at most one $1$ and one $-1$ per column, so it defines the incidence matrix for a directed graph if we add a source and sink node to fill out columns with a single entry; $R_s^c$ is equivalent to a min-flow cost problem over this graph.

\subsubsection{Polyhedral Refinement} \label{app:refine}

We now show how to tighten Dracula's LP relaxation by adding additional constraints to $Q_C$ to shrink it closer to $Q$. If every time we see a string $s$ it is followed by the character $\alpha$ (in a given corpus), the strings $s$ and $s\alpha$ belong to the same \emph{equivalence class}; the presence of $s\alpha$ conveys the same information as the presence of $s$. Importantly, the theory of suffix trees shows that all substrings of a document corpus can be grouped into at most $2n-1$ equivalence classes (\cite{DBLP:books/cu/Gusfield1997}) where $n$ is the word count of the corpus. We always take equivalence classes to be inclusion-wise maximal sets and say that equivalence class $\varepsilon \subset \mathcal{S}$ appears at a location if any (i.e. all) of its members appear at that location. We prove the following theorem below. This theorem verifies common sense and implies that, when the pointer costs do not favor any particular string in $\varepsilon$, adding the constraint $\sum_{s\in\varepsilon}t_s \leq 1$ to the LP relaxation to tighten $\mathcal{Q}_C$ will not remove any binary solutions. 
\begin{thm}
Let $\Omega$ denote the set of all equivalence classes in corpus $\mathcal{C}$ and suppose that all costs are non-negative and $\forall \varepsilon \in \Omega, \forall z\in\mathcal{S}, \forall s,x\in\varepsilon$, the dictionary costs $d_s=d_x$ are equal, the pointer costs $c^z_p=c^z_q$ ($\hat{c}^z_p=\hat{c}^z_q$) are equal when $p=(l,s)$ and $q=(l,x)$, and $c^s_p=c^{x}_q$ ( $\hat{c}^s_p=\hat{c}^{x}_q$) whenever pointers $p=q=(l,h)$ refer to the same location and use the same string (or character) $h$. Then there is an optimal compression $\mathfrak{D} = (S,P\hat{P})$ in which $S$ contains \emph{at most} one member of $\varepsilon$.
\end{thm}
\begin{proof}
Suppose that the conditions for theorem 1 hold, let $\varepsilon$ be an equivalence class, let $\mathfrak{D}=(S,P,\hat{P})$ be an optimal compression, and suppose for the sake of contradiction that $s_1,s_2\in \varepsilon$ are both included in the optimal dictionary. Without loss of generality we assume that $|s_1| < |s_2|$. Consider first document pointer $p$ which uses $s_1$ for document $D_k$. By assumption there is another pointer $q$ which uses $s_2$ in the same location and $c_p^{D_k}= c_q^{D_k}$ so we are indifferent in our choice. We thereby may replace all document pointers that use $s_1$ with equivalent ones that use $s_2$ without changing the objective value.
 
Consider next the usage of $s_1$ to construct higher order dictionary elements. We must be careful here since if some dictionary element $s_3$ is in the optimal dictionary $S$ and can be expressed as $s_3=zs_1$ for some string $z$ then we may not use $s_2$ in place of $s_1$ since it would lead to a different dictionary string. The key step here is to realize that $s_3$ must belong to the same equivalence class as string $zs_2$ and we can use $zs_2$ in place of $s_3$ in all documents. If $s_3$ is itself used to construct higher order dictionary elements, we can apply the same argument for $s_2$ to $zs_2$ in an inductive manner. Eventually, since our text is finite, we will reach the highest order strings in the dictionary, none of whose equivalence class peers construct any other dictionary n-grams. Our earlier argument shows that we can simply take the longest of the highest order n-grams that belong to the same equivalence class. Going back to $s_3$, we note that our assumptions imply that the cost of constructing $zs_2$ is identical to the cost of constructing $s_3$ so we may safely replace $s_3$ with $zs_2$.
The only remaining place where $s_1$ may be used now is to construct $s_2$. However, our assumptions imply that the cost of constructing $s_1$ ``in place'' when constructing $s_2$ is the same. By eliminating $s_1$ we therefore never can do worse, and we may strictly improve the objective if $t_{s_1}>0$ or $s_1$ is used to construct $s_2$ and its pointer cost is non-zero. QED.
\end{proof} 

\subsection{Weighted Diffusion}\label{app:diff}
When $G$ is generated from the relaxation of Dracula and $t\in(0,1]^{|S|}$ are the dictionary coefficients, any $s\in S$ with $t_s < 1$ will have $G_{sz} \leq t_s \forall z \in S$. In order to prevent overly attenuating the diffusion we may wish to normalize row $s$ in $G$ by $t_s^{-1}$ for consistency. We note that a variety of other weightings are also possible to different effects. For example, weighting $G$ by a scalar $\rho \geq 0$ attenuates or enhances the entire diffusion process and mitigates or enhances the effect of features the farther away they are from directly constructing any feature directly used in the documents.

\subsection{Proof of Path Theorem}\label{app:path}
The fundamental theorem of linear programming states that for any $c\in \mathbb{R}^d, S(c,Q) \equiv \text{arg min}_{x \in Q}\, x^Tc(\alpha)\in\mathscr{F}(Q)$ since $Q$ has non-empty interior and is therefore non-empty. We will use a construction known as the normal fan of $Q$, denoted by $\mathcal{N}(Q)$, that partitions $\mathbb{R}^d$ into a finite set of polyhedral cones pertaining to (linear) objectives for which each face in $\mathscr{F}(Q)$ is the solution set. We begin with some helpful definitions.

\begin{def}\label{def:part}
A \emph{partition} $P \subset 2^X$ of a set $X$ is any collection of sets satisfying $\bigcup_{p\in P}p = X$ and $\forall p,q\in P$ $p\neq q$ implies $p \cap q = \emptyset$.
\end{def}
\begin{def}\label{def:relint}
The \emph{relative interior} of a convex set $X \subset \mathbb{R}^d$, denoted by $\myrelint  X$, is the interior of $X$ with respect to its affine hull. Formally, $\myrelint X = \left\{x \in X \mid \exists \varepsilon > 0, B(x,\varepsilon) \cap \text{aff}X \subset X \right\}$.
\end{def}
The following definition is taken from (\cite{lu2008normal}):
\begin{def}\label{def:fan}
A \emph{fan} is a finite set of nonempty polyhedral convex cones in $\mathbb{R}^d$, $\mathcal{N}=\{N_1,N_2,\dots,N_m\}$, satisfying:
\begin{enumerate}
\item any nonempty face of any cone in $\mathcal{N}$ is also in $\mathcal{N}$,
\item any nonempty intersection of any two cones in $\mathcal{N}$ is a face of both cones.
\end{enumerate}
\end{def}
This definition leads to the following lemma, which is adapted from (\cite{lu2008normal}):
\begin{lem}\label{lem:fan}
Let $\mathcal{N}$ be a fan in $\mathbb{R}^d$ and $S = \bigcup_{N \in \mathcal{N}} N$ the union of its cones.
\begin{enumerate}
\item If two cones $N_1,N_2\in\mathcal{N}$ satisfy $(\textnormal{relint} N_1) \cap N_2 \neq \emptyset$ then $N_1 \subset N_2$,
\item The relative interiors of the cones in $\mathcal{N}$ partition $S$, i.e. $\bigcup_{N \in \mathcal{N}} \textnormal{relint}N = S$.
\end{enumerate}
\end{lem}
Lemma \ref{lem:fan} is subtle but important as it contains a key geometric insight that allow us to prove our theorem. Next, let $Q\subset \mathbb{R}^d$ be a bounded polyhedron with vertex set $V$ and nonempty interior, i.e. whose affine hull is $d$-dimensional. For any face $F \in \mathscr{F}(Q)$ define $V(F) = F \cap V$ to be the vertices of $F$ and $N_F = \left\{ y \in \mathbb{R}^d \mid \forall x\in F, \forall z \in Q, y^Tx \leq y^Tz \right\}$ to be the normal cone to $F$. That $N_F$ is a (pointed) polyhedral cone follows from noting that it can be equivalently expressed as a finite collection of linear constraints involving the vertices of $F$ and $Q$: $N_F = \left\{ y \in \mathbb{R}^d \mid \forall x\in V(F), \forall z \in V, y^Tx \leq y^Tz \right\}$. The \emph{normal fan} for $Q$, $\mathcal{N}(Q) = \{N_F\}_{F \in \mathscr{F}(Q)}$, is defined to be the set of all normal cones for faces of $Q$. Noting that $Q$ is bounded and therefore has a recession cone of $\{0\}$, the following Lemma is implied by Proposition 1 and Corollary 1 of (\cite{lu2008normal}):
\begin{lem}\label{lem:normal_fan}
Let $\mathcal{N}(Q)$ be the normal fan of a bounded polyhedron $Q$ with non-empty interior in $\mathbb{R}$. Then
\begin{enumerate}
\item $\mathcal{N}(Q)$ is a fan,
\item for any nonempty faces $F_1,F_2\in\mathscr{F}(Q)$, $F_1 \subset F_2$ iff $N_{F_1} \supset N_{F_2}$,
\item $\bigcup_{F \in \mathscr{F}(Q)} \textnormal{relint}N_F = \mathbb{R}^d$,
\item every nonempty face $F\in\mathscr{F}(Q)$ satisfies $\textnormal{relint}N_F=\left\{y\in \mathbb{R}^d\mid F = S(y,Q)\right\}$.
\end{enumerate}
\end{lem}

We will also makes use of the following two results. The first is implied by Theorem 2.7, Corollary 2.14,  and Problem 7.1 in \cite{ziegler1995lectures}:
\begin{lem}\label{lem:dim}
Let $Q\subset \mathbb{R}^d$ be a bounded polyhedron with nonempty interior, $F \in \mathscr{F}(Q)$, and $N_F$ the normal cone to $F$. Then $\textnormal{dim}\, F + \textnormal{dim}\,  N_F = d$.
\end{lem}
The second Lemma states a kind of neighborliness for the cones in $\mathcal{N}(Q)$:
\begin{lem}\label{lem:nbr}
Let $Q\subset \mathbb{R}^d$ be a bounded polyhedron with nonempty interior. For any $N\in\mathcal{N}(Q)$ and $x \in \textnormal{relint}\,  N$ there exists a $\delta > 0$ such that for any $y \in B(x,\delta)$ there is a $N' \in \mathcal{N}(Q)$ with $y \in \textnormal{relint}\, N'$ and $N \subset N'$. 
\end{lem}
\begin{proof}
Let $N\in\mathcal{N}(Q)$ and $x \in N$ be given. We say that $N'\in\mathcal{N}(Q)$ \emph{occurs within} $\delta$ (for $\delta > 0)$ if there is some $y \in B(x, \delta)$ with $y \in \myrelint  N'$. Now suppose that there is an $N'\in\mathcal{N}(Q)$ that occurs within $\delta$ for all $\delta >0$. Since $N'$ is a closed convex cone it must be that $x \in N'$ so we may conclude from Lemma \ref{lem:fan} that $N\subset N'$. Next, let $\mathcal{M}$ be the set of cones in $\mathcal{N}(Q)$ which do not contain $N$ and suppose that for all $\delta > 0$ there is some $N' \in \mathcal{M}$ that occurs within $\delta$. Since $|\mathcal{M}|$ is finite, this is only possible if there is a cone $N'\in\mathcal{M}$ that occurs within $\delta$ for all $\delta > 0$. However, this leads to a contradiction since $N'$ must contain $N$ so the Lemma follows.
\end{proof}

We are now ready to prove our main theorem which is restated below with $S(c,Q) = \text{arg min}_{x \in Q}\, x^Tc(\alpha)$ for simplicity.
\begin{thm}
Let $Q\subset \mathbb{R}^d$ be a bounded polyhedron with nonempty interior and $c:[0,1]\rightarrow\mathbb{R}^d$ a continuous function. Then for some $N\in\mathbb{Z}_+\cup \{\infty\}$ there exists a countable partition $\Gamma = \left\{\gamma_i\right\}_{i = 0}^{N}$ of $[0,1]$ with corresponding faces $F_i\in\mathscr{F}(Q)$ satisfying $F_i \neq F_{i+1}$ and  $F_{i}\cap F_{i+1} \neq \emptyset$ and $F_i = S(c(\alpha), Q)\, \forall \alpha \in \gamma_i$. Moreover, $F_i$ never has the same dimension as $F_{i+1}$ and the boundary between $\gamma_i,\gamma_{i+1}$ is $)[$ iff $\textnormal{dim}\,F_i < \textnormal{dim}\,F_{i+1}$ and $]($ otherwise.
\end{thm}

\begin{proof}
For ease of notation let $f(x) = S(c(x),Q)$ and for $k = 0,\dots,d$ define $\omega_k = \left\{x \in [0,1]\mid \mydim N_{f(x)} \geq k\right\}$ to be the set of all arguments to $c$ whose normal cone is at least $k$-dimensional. Moreover, for any $x \in [0,1]$ define $\sigma(x) = \left\{y\in [0,x]\mid \forall z\in [y,x], f(x) = f(z) \right\} \cup \left\{y\in [x,1]\mid \forall z\in [x,y], f(x) = f(z) \right\}$ to be the largest contiguous set containing $x$ over which $f$ remains constant and let $m(x) = \myinf \sigma(x)$ and $M(x) = \mysup \sigma(x)$ be its infinimum and supremem, respectively. The proof follows by  induction on $k=d,d-1,\dots,0$ with the inductive hypothesis that for some $N_k\in\mathbb{Z}_+\cup \{\infty\}$ there exists a countable partition $\Gamma^k = \left\{\gamma_i^k\right\}_{i = 0}^{N_k}$ of $\omega_k$ with corresponding faces $F_i^k\in\mathscr{F}(Q)$ satisfying $F_i^k = S(c(\alpha),Q) \, \forall \alpha \in \gamma_i^k$.

Base case ($k = d$): Let $x \in \omega_d$ so that $\sigma(x) \subset \omega_d$. Since $N_{f(x)}$ is $d$-dimensional, $\myint N_{f(x)} = \myrelint N_{f(x)}$ so continuity of $c$ implies that $\sigma(x)$ is a (non-empty) open interval with $m(x) < M(x)$. It follows that $\Gamma^k = \left\{\sigma(x)\mid x \in \omega_d \right\}$ defines a partition of $\omega_d$ into a set of open intervals. Each interval contains (an infinite number) of rational numbers, and we see that $\Gamma^k$ is countable by assigning to each interval a rational number that it contains.

Inductive step: Let $x \in \omega_k \backslash \omega_{k+1}$. There are two cases to consider. If $m(x) < M(x)$ then $(m(x),M(x)) \subset \sigma(x)$ contains a rational number. Thus, the set $\Gamma^k_o = \left\{\sigma(x)\mid x \in \omega_k \backslash \omega_{k+1}, m(x) < M(x)\right\}$ is countable. Otherwise, if $m(x) = x = M(x)$ then by Lemma \ref{lem:nbr} there is a $\delta > 0$ such that if $y\in B(x,\delta)$ then $N_{f(x)}\subset N_{S(y,Q)}$. Continuity of $c$ implies that there is a $\varepsilon > 0$ for which $c((x - \varepsilon, x + \varepsilon)) \subset B(x, \delta)$ and hence $(x - \varepsilon, x + \varepsilon)\backslash \{x\} \subset \omega_{k+1}$. Assigning to $x$ any rational number in $(x - \varepsilon, x + \varepsilon)$ and letting $\Gamma^k_c = \left\{\sigma(x)\mid x \in \omega_k \backslash \omega_{k+1}, m(x) = M(x)\right\}$, we may appeal to the inductive hypothesis to conclude that $\Gamma^k_c$ is countable. Finally, $\Gamma^k = \Gamma^k_o \cup \Gamma^k_c \cup \Gamma^{k+1}$ is a finite union of countable sets and therefore countable.

Since $\omega_0 = [0,1]$ we have shown that $\Gamma = \Gamma^0$ is a countable partition of $[0,1]$ into intervals over which $f$ is constant. Now consider two consecutive intervals $\gamma_i,\gamma_{i+1}\in\Gamma$ and let $M$ be the supremum of $\gamma_i$. If $M\notin\gamma_i$ then since cone $N_{F_i}$ is closed, $c(M) \in N_{F_i}$. Since $c(M) \in \myrelint  N_{F_{i+1}}$ by assumption, it follows that $N_{F_{i+1}}$ is a proper subset of $N_{F_i}$ and hence that $F_i$ is a proper subset of $F_{i+1}$. Otherwise, if $M \in \gamma_i$ then the continuity of $c$ and Lemma \ref{lem:nbr} imply that $N_{F_{i}}$ is a proper subset of $N_{F_{i+1}}$ so $F_{i+1}$ is a proper subset of $F_{i}$. In either case $F_{i}\cap F_{i+1}\neq \emptyset$ and Lemma \ref{lem:dim} implies the dimensionality result of our Theorem.
\end{proof}

\end{document}